\documentclass[conference]{IEEEtran}
\IEEEoverridecommandlockouts
\usepackage{cite}
\usepackage{hyperref} 
\usepackage{amsmath,amssymb,amsfonts}
\usepackage[linesnumbered,ruled,vlined]{algorithm2e}
\usepackage{algorithmic}
\usepackage{comment}
\usepackage{graphicx}
\usepackage{textcomp}
\usepackage{multirow}
\usepackage{xcolor}
\usepackage{float}
\usepackage{tikz}
\usepackage{array}
\usepackage{pgfplots}
\pgfplotsset{compat=1.18}
\usepackage{subcaption}
\def\BibTeX{{\rm B\kern-.05em{\sc i\kern-.025em b}\kern-.08em
    T\kern-.1667em\lower.7ex\hbox{E}\kern-.125emX}}

    \usepackage{amsthm,enumerate,verbatim} 

\newtheorem{theorem}{Theorem} 
\newtheorem{lemma}{Lemma} 

\newcommand{\kronecker}{\raisebox{1pt}{\ensuremath{\:\otimes\:}}}

\DeclareMathOperator{\rank}{rank}

\begin{document}

\title{Efficient algorithms for the Hadamard decomposition}

\author{\IEEEauthorblockN{Samuel Wertz}
\IEEEauthorblockA{\textit{University of Mons}\\
Mons, Belgium \\
Samuel.Wertz@student.umons.ac.be}
\and
\IEEEauthorblockN{Arnaud Vandaele}
\IEEEauthorblockA{\textit{University of Mons}\\
Mons, Belgium \\
Arnaud.Vandaele@umons.ac.be}
\and
\IEEEauthorblockN{Nicolas Gillis}
\IEEEauthorblockA{\textit{University of Mons}\\
Mons, Belgium \\
Nicolas.Gillis@umons.ac.be}
}

\definecolor{brightpink}{rgb}{1.0, 0.0, 0.5}
\newcommand{\ngc}[1]{{\color{brightpink} (\textbf{NG:} #1)}}
\newcommand{\ngi}[1]{{{\color{brightpink} #1}}}

\definecolor{forestgreen(web)}{rgb}{0.13, 0.55, 0.13}
\newcommand{\av}[1]{{\color{forestgreen(web)} (\textbf{AV:} #1)}}

\maketitle

\begin{abstract}
The Hadamard decomposition is a powerful technique for data analysis and matrix compression, which decomposes a given matrix into the element-wise product of two or more low-rank matrices. In this paper, we develop an efficient algorithm to solve this problem, leveraging an alternating optimization approach that decomposes the global non-convex problem into a series of convex sub-problems. To improve performance, we explore advanced initialization strategies inspired by the singular value decomposition (SVD) and incorporate acceleration techniques by introducing momentum-based updates. Beyond optimizing the two-matrix case, we also extend the Hadamard decomposition framework to support more than two low-rank matrices, enabling approximations with higher effective ranks while preserving computational efficiency. Finally, we conduct extensive experiments to compare our method with the existing gradient descent-based approaches for the Hadamard decomposition and with traditional low-rank approximation techniques. The results highlight the effectiveness of our proposed method across diverse datasets.
\end{abstract}


\section{Introduction}
The general low-rank decomposition problem consists in approximating a given matrix \(X \in \mathbb{R}^{m \times n}\) by a matrix \(\tilde{X}\) of rank \(r < \min(m, n)\). The objective being to find the best approximation, it is common to consider the following optimization problem : 
\begin{equation}
\min_{\tilde{X}} \| X - \tilde{X} \|_F^2 \quad \text{s.t.} \quad \text{rank}(\tilde{X}) \leq r,
\label{eq:general_problem}
\end{equation}
where \( \| . \|_F\) denotes the Frobenius norm.
To enforce the rank constraint on \(\tilde{X}\), a commonly used approach is to write the matrix as \(\tilde{X} = WH\) with \(W \in \mathbb{R}^{m \times r}\) and \(H \in \mathbb{R}^{r \times n}\). Thus, the optimization problem \eqref{eq:general_problem} can be reformulated as follows: 
\begin{equation}
    \label{eq:general_problem_2factors}
    \min_{W, H}{\| X - WH \|_F^2} . 
\end{equation} 
For sufficiently small values of \(r\), specifically \(r < \frac{mn}{m+n}\), this decomposition enables compression of the original matrix. This property makes low-rank decompositions a powerful tool for compression. 
Beyond compression, low-rank factorizations are valuable for their ability to extract meaningful features from data, which has established them as essential techniques in data analysis and machine learning. 
The most famous model used to resolve this low-rank approximation problem is the singular value decomposition (SVD), a fundamental technique in linear algebra. For a given matrix \(X \in \mathbb{R} ^{m \times n}\), the SVD factorizes \(X\) as $X = U\Sigma V^T$, where 
    \(U \in \mathbb{R}^{m \times m}\) is an orthogonal matrix whose columns are the left singular vectors of \(X\),
    \(\Sigma \in \mathbb{R}^{m \times n}\) contains the singular values in nonincreasing order on its diagonal, and 
    \(V \in \mathbb{R}^{n \times n}\) is an orthogonal matrix whose columns are the right singular vectors.
The Eckart–Young–Mirsky theorem proves that using 
\begin{align*}
    W = U(:, :r) \sqrt{\Sigma(:r,:r)}, \quad  
    H = \sqrt{\Sigma(:r,:r)} V^T(:r,:),  
\end{align*}
provides an optimal solution 
for~\eqref{eq:general_problem_2factors}.

\paragraph*{The Hadamard Decomposition}  
Standard low-rank approximations are limited in the rank of \(\tilde{X}\) by \(r\) and cannot exploit element-wise sparsity or structural properties inherent in some datasets. These limitations have motivated the development of the model we will now study. 
The Hadamard product between two matrices of the same dimension $A$ and $B$ (also referred to as the element-wise or component-wise product), denoted by \(A \circ B\), is defined as:  
\begin{equation}
    (A \circ B)_{i,j} = A_{i,j} \cdot B_{i,j} \quad \text{ for all } i,j. \nonumber
\end{equation} 
The Hadamard decomposition seeks to approximate a given matrix \(X\) as the element-wise product of two low-rank matrices, expressed as:  
\begin{equation}
    \tilde{X} = (W_1 H_1) \circ (W_2 H_2), \nonumber
\end{equation}
where \(W_1, W_2 \in \mathbb{R}^{m \times r}\) and \(H_1, H_2 \in \mathbb{R}^{r \times n}\).   
This leads to the following optimization problem:  
\begin{equation}
    \label{eq:hadamard_problem}
    \min_{W_1, H_1, W_2, H_2} \| X - (W_1 H_1) \circ (W_2 H_2) \|_F^2 . 
\end{equation} 
First introduced in \cite{b1} with two matrices of rank $r$, this model is able to reach a maximal rank of \(r^2\) for \(\tilde{X}\). In fact, it can be shown that the component-wise product of two rank-$r$ matrices has rank at most $r^2$. To solve this problem, \cite{b2} proposed a straightforward approach based on alternating gradient descent, optimizing one factor at a time, $W_1$, $H_1$, $W_2$ or $H_2$, while keeping the others fixed. This method iteratively minimizes the reconstruction error and has demonstrated effectiveness in approximating complex datasets. However, this solution has certain limitations, including the reliance on a manually set fixed step size for the gradient descent and the random initialization of matrices, among others.

\paragraph*{Contribution and outline} In this work, we 
(i)~propose a block-coordinate descent (BCD) algorithm for significantly more efficient updates of the factor matrices in the Hadamard decomposition,
(ii)~investigate improved initialization strategies, and 
(iii)~explore acceleration techniques to speed up convergence and handle larger datasets more effectively. 

The remainder of this paper is structured as follows. 
Section~\ref{sec:proposedalg} details the proposed BCD optimization framework, initialization strategies, and acceleration techniques. 
Section~\ref{sec:morefact} explains how our strategies can be adapted when the Hadamard product of more than two matrices is used to approximate $X$. 
Experimental results are presented in Section~\ref{sec:experiments}, where we compare our approach with the SVD and other baselines. Finally, Section~\ref{sec:concl} concludes the paper and outlines directions for future research.

\section{BCD for the Hadamard Decomposition} \label{sec:proposedalg}
Since \eqref{eq:hadamard_problem} is non-convex, 
we adopt an alternating optimization approach, wherein the four factor matrices, $W_1$, $H_1$, $W_2$ and $H_2$, are updated sequentially; see Alg.~\ref{alg:BCD}.  
The reason for this choice is that, when three of the factors are fixed, the resulting sub-problem in the last factor is convex and can be solved efficiently; 
see below.  
Moreover, due to the symmetry of the problem, the four factors can be updated using the same function, which we denote \texttt{UpdFact}; see Alg.~\ref{alg:updateH}. 
\begin{algorithm}[h]
    \caption{Alternating scheme for Hadamard decomposition problem}
    \label{alg:BCD}
    \SetKwFunction{FMain}{BCD}
    \SetKwProg{Fn}{Function}{:}{}
    \Fn{\FMain{$X$, $maxiter$}}{
        $W_1, H_1, W_2, H_2 \gets initialization$\;
        \For{$i \gets 1$ \KwTo $maxiter$}{
            $H_2 \gets \texttt{UpdFact}(X, W_1, H_1, W_2, H_2)$\;
            $W_2 \gets \texttt{UpdFact}(X^T, H_1^T, W_1^T, H_2^T, W_2^T)$\;
            $H_1 \gets \texttt{UpdFact}(X, W_2, H_2, W_1, H_1)$\;
            $W_1 \gets \texttt{UpdFact}(X^T, H_2^T, W_2^T, H_1^T, W_1^T)$\;
        }
        \Return $W_1, H_1, W_2, H_2$\;
    }
\end{algorithm}
\begin{algorithm}[ht!]
    \caption{\texttt{UpdFact} function}
    \label{alg:updateH}
    \SetKwFunction{FMain}{\texttt{UpdFact}}
    \SetKwProg{Fn}{Function}{:}{}
    
    \Fn{\FMain{$X, W_1, H_1, W_2, H_2$}}{ 
        \For{$j \gets 1$ \KwTo $n$}{
            $H_2(:,j) \gets \texttt{hadLS}(
                W_1H_1(:,j), W_2, X(:,j))$\;
        }
        \Return $H_2$
    }
\end{algorithm}

Let us now discuss the update of one factor,  and consider w.l.o.g.\ the update of \(H_2\). The corresponding minimization problem is convex and column-separable, as it can be expressed as the sum of $n$ independent problems :  
\begin{equation}
    \label{eq:column_separation}
    \min_{H_2} \sum_{j=1}^n \| X(:,j) - (W_1 H_1(:,j)) \circ (W_2 H_2(:,j)) \|_F^2.
\end{equation} 
To solve \eqref{eq:column_separation}, \texttt{UpdFact} (see Alg.~\ref{alg:updateH}) solves the \(n\) column-wise subproblems one after the other using the \texttt{hadLS} subroutine.

The \texttt{hadLS} function solves the least squares problem corresponding to a single column of $H_2$, denoted $H_2(:,j)$. In order to simplify the presentation for this subproblem, we introduce the notation summarized in Table~\ref{tab:notations1}. 
\begin{table}[h]
    \centering
    \begin{tabular}{l|l}
        \hline
        \textbf{Notation} & \textbf{Description} \\
        \hline
        \(x = H_2(:,j)\) & The optimization variable \(H_2(:,j)\) \\
        \(A = W_2\) & The matrix \(W_2\), serving as a factor in the decomposition \\
        \(b = X(:,j)\) & The \(j\)-th column of the input matrix \(X(:,j)\) \\
        \(s = W_1 H_1(:,j)\) & The Hadamard product term \(W_1 H_1(:,j)\) \\
        \hline
    \end{tabular}
    \vspace{3pt}
    \caption{Least squares problem notation.}
    \label{tab:notations1}
\end{table} 
This allows us to write the column-wise optimization problem to update $H_2(:,j)$ as follows: 
\begin{equation}
    \min_{x} f(x) := \| s \circ (Ax) - b \|_2^2 . 
    \label{eq:obj}
\end{equation} 
This Hadamard least squares problem is quadratic and unconstrained, the gradient and the Hessian are given by:
\begin{align}
    \nabla f(x) &= A^T \left( (s \circ (Ax) - b) \circ s \right) \notag \\
                &= (A^T \, \text{Diag}(s^2) \, A)x - A^T (s \circ b) , \label{eq:gradient} \\
    \nabla^2 f(x) &= A^T \, \text{Diag}(s^2) \, A .\label{eq:hessian}
\end{align}

\subsection{Solving hadLS~\eqref{eq:obj}} 

We explore two strategies to solve hadLS~\eqref{eq:obj}. 

\subsubsection{Gradient descent (GD) based iterative methods} 

We first use an iterative gradient descent using \eqref{eq:gradient}. To have convergence, the stepsize should be chosen carefully. We implemented two variants: one that calculates the stepsize using the Lipschitz constant, and another that computes the optimal stepsize.

\subsubsection{Exact solution} 

The optimal solution for \(x\) can be computed as the solution of a quadratic, unconstrained problem. Setting \(\nabla f(x) = 0\) yields  
    $\left( A^T \, \text{Diag}(s^2) \, A \right) x =  A^T (s \circ b)$. 
This requires solving a linear system in $r$ variables, with costs  $\mathcal{O}(r^3)$. This exact solution eliminates the need for (explicit) iterative updates, taking advantage of higly-efficient libraries to solve least squares problem, and provides an optimal solution for the subproblems; see Alg.~\ref{alg:analytic}. 
\begin{algorithm}[h]
    \caption{\texttt{hadLS} with Exact Resolution}
    \label{alg:analytic}
    \SetKwFunction{Fsecond}{\texttt{hadLS}}
    \SetKwProg{Fn}{Function}{:}{}
    \Fn{\Fsecond{\texttt{s}, \texttt{A}, \texttt{b}}}{
        $H \gets A^T \, \text{Diag}(s^2) \, A$\; 
        $d \gets A^T (s \circ b)$\;
        $x \gets H$\textbackslash $d$\;
        \Return $x$\;
    }
\end{algorithm}

\subsubsection{Complexity analysis}  
For the three versions of the algorithm, we first compute the quantities $H$ in $\mathcal{O}(mr^2)$ and $d$ in $\mathcal{O}(mr)$. We can also compute the Lipschitz constant by computing the singular values of $A^T \, \text{Diag}(s)$ in $\mathcal{O}(mr^2)$. Solving the linear system costs $\mathcal{O}(r^3)$ and computing the gradient only cost $\mathcal{O}(r^2)$ when $H$ and $d$ are already constructed. In order to take advantage of this difference, we perform an appropriate number of inner iterations for the gradient descent using the same pre-computed $H$ and $d$. Since the optimal step size costs only $\mathcal{O}(r^2)$, it can be computed at each inner iteration.

\subsection{Initializations}
\label{sec:init}

In this section, we propose several initialization methods, including well-known techniques from machine learning and approaches inspired by the methods presented in \cite{b4}.

\subsubsection{Xavier Initialization}

Xavier initialization is widely used in machine learning to initialize network weights. It aims to maintain the variance of activations throughout the layers. 

- Uniform Xavier uses a uniform distribution between \(-\beta\) and \(\beta\), with $\beta = \sqrt{6/(m+n)}$.

- Normal Xavier samples from a normal distribution centered at 0 with a standard deviation of:
    $\sigma = \sqrt{{2}/{(m+n)}}$. 

These methods are especially effective for balancing the scale of weights across the network.

\subsubsection{SVD-Based Initialization}
\label{sec:svd-init}

As the SVD leads to the best rank-$r$ approximation with a single matrix, we propose an efficient initialization by computing two SVDs. The idea is as follows: Let us define $M = \sqrt{|X|}$ as the element-wise square root of \(|X|\): 
$M(i,j) = \sqrt{|X(i,j)|}$ for all  $i,j$, and  $M_{\text{signed}} = M \circ \text{sign}(X)$ as the same matrix but with the signs of the entries of $X$ taken into account. This implies that $X = M \circ M_{\text{signed}}$. Then we initialize $W_1 H_1$ as the best rank-$r$ approximation of $M$, and $W_2 H_2$ as the best rank-$r$ approximation of $M_{\text{signed}}$ so that 
\[ 
X = M \circ M_{\text{signed}}
\approx (W_1 H_1) \circ  (W_2 H_2) 
\]
is a meaningful Hadamard decomposition of $X$, capturing both magnitude and sign information from \(X\).







\subsubsection{K-Means Initialization}

K-Means-based initialization uses the same scheme as the SVD-based one where we replace the SVDs by k-means clustering to initialize the factors: 
The \(r\) cluster's centroids obtained from the K-Means algorithm are used to initialize \(W_i\) and the associations between data points and centroids provide the initialization of \(H_i\).

\subsubsection{Optimal scaling} Given any initialization, $\tilde X = (W_1 H_1) \circ  (W_2 H_2)$, it can be improved by scaling it optimally by solving \( \min_{\alpha} || X - \alpha \tilde{X} || \) with optimal solution \( \alpha^* = \frac{\langle \tilde{X}, X \rangle}{|| \tilde{X} ||_F^2 } \). We apply this trick to all initializations by multiplying the first factor by sign$(\alpha^*)(\sqrt[4]{|\alpha^*|})$ and the three others by \(\sqrt[4]{|\alpha^*|}\). 


Fig.~\ref{fig:init} presents the average relative error \(e(t) = \frac{\|X - \tilde{X}(t) \|}{\|X\|}\), where $t$ is the iteration index,  computed over 10 trials for the 5 initialization methods on synthetic full-rank datasets of size \(100 \times 100\) generated by sampling the standard normal distribution. 
This error is normalized by subtracting the minimum error \(e_{min}\) obtained across all methods and initializations, and then dividing by the difference between the initial error \(e(0)\) and the minimum error. 
\begin{figure}[ht]
        \centering
        \begin{tikzpicture}[scale=0.8]
            \begin{axis}[
                xlabel={Iterations},
                ylabel = {$\frac{e(t) - e_{min}}{e(0) - e_{min}}$},
                ymode=log,
                grid=both,
                legend style={at={(0.3, 0.45)}, anchor=north, legend columns=1, font=\large},
                tick label style={font=\large},
                xlabel style={font=\large}, 
                ylabel style={font=\LARGE},
            ]
                
                \addplot[color=red, mark=o] table [x=x1, y=y1, col sep=space] {Graphes/init_comp_opti_scale.txt};
                \addlegendentry{Random}
    
                \addplot[color=blue, style=dashed] table [x=x2, y=y2, col sep=space] {Graphes/init_comp_opti_scale.txt};
                \addlegendentry{Uniform Xavier}
    
                \addplot[color=green, mark=square*] table [x=x3, y=y3, col sep=space] {Graphes/init_comp_opti_scale.txt};
                \addlegendentry{Normal Xavier}
    
                \addplot[color=violet, mark=triangle*] table [x=x4, y=y4, col sep=space] {Graphes/init_comp_opti_scale.txt};
                \addlegendentry{SVD}
    
                \addplot[color=orange, mark=star] table [x=x5, y=y5, col sep=space] {Graphes/init_comp_opti_scale.txt};
                \addlegendentry{K-means}
            \end{axis}
        \end{tikzpicture}
    \caption{Evolution of the normalized average error  for the five initializations for BCD with exact resolution.}
    \label{fig:init}
\end{figure}
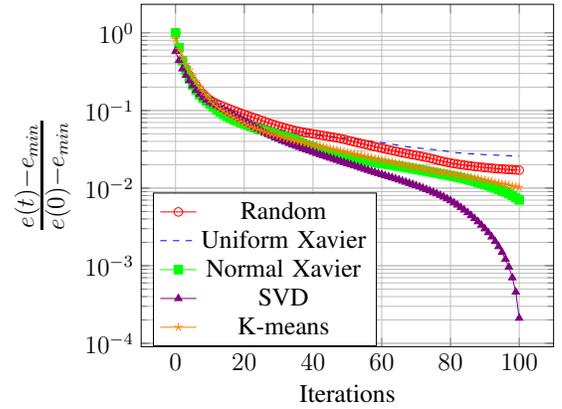
We observe that the SVD initialization provides better initial points and allows convergence to better minima. 

\subsection{Acceleration by Adding Momentum}

To accelerate the convergence of our algorithm, 
we propose a momentum-based approach by extrapolating the updated matrices at each iteration. Specifically, the following extrapolation step is added at the end of the \texttt{UpdFact} function:
\[
    H_2 = H_2 + \beta_k (H_2 - H_{2\text{old}}), 
\]
where \(H_{2\text{old}}\) is the value of \(H_2\) from the previous iteration. 


For the choice of the extrapolation parameter \(\beta_k\), we rely on the strategy proposed by Ang et al.~\cite{b6}. The extrapolation parameter is dynamically updated according to the improvement of the objective function at each iteration. The procedure is described in Alg.~\ref{alg:beta}. 
The method uses parameters \(1 \leq \tilde{\gamma} \leq \gamma \leq \eta\), \(\beta_0 \in [0, 1]\) and initializes \(\tilde{\beta} = 1\). 
\begin{algorithm}[h]
    \caption{Update of the Momentum Coefficient \(\beta_k\)}
    \label{alg:beta}
    \SetKwFunction{Fsecond}{\texttt{betaUpdate}}
    \SetKwProg{Fn}{Function}{:}{}
    \Fn{\Fsecond{$\beta_k, \tilde{\beta}, \tilde{\gamma}, \gamma, \eta$}}{
        \If{The error decreased at iteration \(k\)}{
            $\beta_{k+1} \gets \min(\tilde{\beta}, \gamma \beta_k)$;
            $\tilde{\beta} \gets \min(1, \tilde{\gamma} \tilde{\beta})$\;
        }
        \Else{
            $\beta_{k+1} \gets \frac{\beta_k}{\eta}$; $\tilde{\beta} \gets \beta_k$\;
        }
        \Return $\beta_{k+1}, \tilde{\beta}$\;
    }
\end{algorithm}
Momentum-based acceleration can significantly speed up convergence, particularly in settings where the updates oscillate or slow down near local minima.
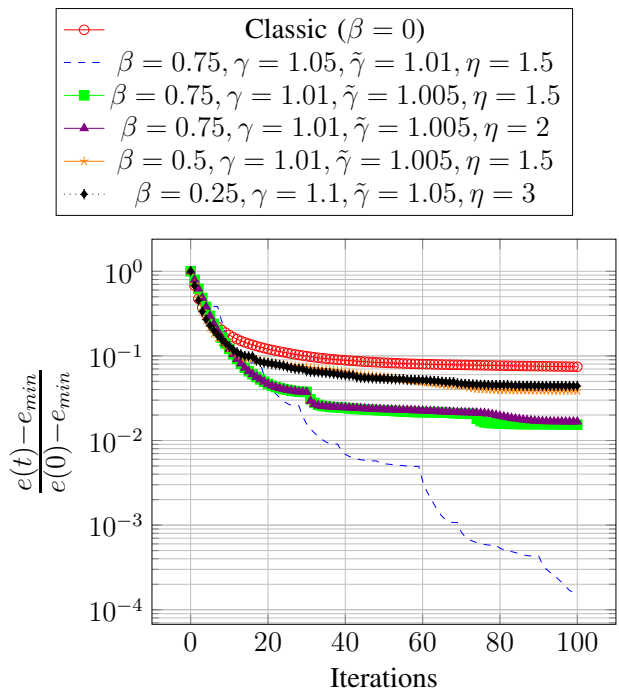
\begin{figure}[ht]
        \centering
        \begin{tikzpicture}[scale=0.9]
            \begin{axis}[
                xlabel={Iterations},
                ylabel = {$\frac{e(t) - e_{min}}{e(0) - e_{min}}$},
                ymode=log,
                grid=both,
                legend style={at={(0.35, 1.6)}, anchor=north, legend columns=1, font=\large},
                tick label style={font=\large},
                xlabel style={font=\large}, 
                ylabel style={font=\LARGE}
            ]
                
                \addplot[color=red, mark=o] table [x=x6, y=y6, col sep=space] {Graphes/acceleration_comparison.txt};
                \addlegendentry{Classic ($\beta = 0$)}
    
                \addplot[color=blue, style=dashed] table [x=x1, y=y1, col sep=space] {Graphes/acceleration_comparison.txt};
                \addlegendentry{$\beta = 0.75, \gamma = 1.05, \tilde{\gamma} = 1.01, \eta = 1.5$}
    
                \addplot[color=green, mark=square*] table [x=x2, y=y2, col sep=space] {Graphes/acceleration_comparison.txt};
                \addlegendentry{$\beta = 0.75, \gamma = 1.01, \tilde{\gamma} = 1.005, \eta = 1.5$}
    
                \addplot[color=violet, mark=triangle*] table [x=x3, y=y3, col sep=space] {Graphes/acceleration_comparison.txt};
                \addlegendentry{$\beta = 0.75, \gamma = 1.01, \tilde{\gamma} = 1.005, \eta = 2$}

                \addplot[color=orange, mark=star] table [x=x4, y=y4, col sep=space] {Graphes/acceleration_comparison.txt};
                \addlegendentry{$\beta = 0.5, \gamma = 1.01, \tilde{\gamma} = 1.005, \eta = 1.5$}
                
                \addplot[color=black, style=dotted, mark=diamond*] table [x=x5, y=y5, col sep=space] {Graphes/acceleration_comparison.txt};
                \addlegendentry{$\beta = 0.25, \gamma = 1.1, \tilde{\gamma} = 1.05, \eta = 3$}

            \end{axis}
        \end{tikzpicture}
    \caption{Impact of extrapolation on BCD with exact resolution.} 
    \label{fig:momentum_effect}
\end{figure}

Fig.~\ref{fig:momentum_effect} uses the same settings as   Fig.~\ref{fig:init} to compare various choices of parameters. We observe that every combination of parameters converges on average faster than the non-accelerated version, and the best set of parameters, for this experiment, is ($\beta = 0.75, \gamma = 1.05, \tilde{\gamma} = 1.01, \eta = 1.5$). 

\section{More than two matrices in the decomposition} \label{sec:morefact}

The Hadamard decomposition in \eqref{eq:hadamard_problem} can be generalized to include more than two low-rank matrices: 
\begin{equation}
    \min_{\{W_i, H_i\}_{i=1}^p} \| X - (W_1 H_1) \circ (W_2 H_2) \circ ... \circ (W_p H_p) \|_F^2 , 
\end{equation}
where \(p\) denotes the number of low-rank matrices. 
This extension is motivated by the (potential) increase of the rank of the approximation, which can be up to \(\prod_{i=1}^p r_i\) where each low-rank matrix has rank $r_i = \rank(W_i H_i)$, and hence  provides  more expressiveness in approximating large matrices. 
Let us illustrate this with the following theorem. 
\begin{theorem} \label{th:identity}
Let $R = \sum_{i=1}^p r_i$ be the budget\footnote{We call this the budget because two decompositions with the same budget have the same number of parameters, in the factors $(W_i, H_i)$'s, namely $\sum_{i=1}^p r_i(m+n) = R (m+n)$.} of an Hadamard decomposition with $p$ matrices of rank $r_i$. 
Such a decomposition can reconstruct exactly the $n$-by-$n$ identity matrix for any  \vspace{-0.1cm} 
\begin{equation*}
   n \leq N = \left\{ 
    \begin{array}{cl}
       3^k  & \text{ if $R = 3k$ for $k\in\mathbb{N}$},  \\
       4 \cdot 3^{k-1} & \text{ if $R = 3k+1$ for $k\in\mathbb{N}$},  \\
       2 \cdot 3^{k} & \text{ if $R = 3k+2$ for $k\in\mathbb{N}$}.  \\
    \end{array} 
    \right. 
\end{equation*}
In such decompositions, $r_i \in \{2,3\}$ for all $i$. 
    
\end{theorem}
\begin{proof}
    See Appendix~\ref{sec:appendixB}. 
\end{proof}
Having more than two low-rank matrices requires some changes in BCD; see Alg.~\ref{alg:multiple}. 
\begin{algorithm}[!htbp]
    \caption{Hadamard decomposition with $p$ factors}
    \label{alg:multiple}
    \SetKwFunction{FMain}{\texttt{BCDmultiple}}
    \SetKwProg{Fn}{Function}{:}{}
    \Fn{\FMain{$X$, $maxiter$}}{
        $W \gets [\;]$; $H \gets [\;]$ \tcp*{Empty lists}  
        \For{$i \gets 1$ \KwTo $p$}{
            $W(i), H(i) \gets initialization()$ \;
        }
        \For{$k \gets 1$ \KwTo $maxiter$}{
            \For{$i \gets 1$ \KwTo $p$}{
                $P(s,t) = 1$ $\forall s,t$ \;
                \For{$j \gets 1$ \KwTo $p$}{
                    \If{$j \neq i$}{
                        $P \gets P \circ W(j) H(j)$ \;
                    }
                }
                 \small $H(i) \gets \texttt{UpdtFact}(X, P, W(i), H(i))$\;
                 \small $W(i) \gets \texttt{UpdtFact}(X^T, P^T, H(i)^T, W(i)^T)$\; \normalsize 
            }
        }
        \Return W, H\;
    }

    \SetKwFunction{Fsecond}{\texttt{UpdtFact}}
    \SetKwProg{Fn}{Function}{:}{}
    \Fn{\Fsecond{$X, WH_1, W_2, H_2$}}{
        \For{$j \gets 1$ \KwTo $n$}{
            $H_2(:,j) \gets \texttt{hadLS} (WH_1(:,j) ,W_2, X(:,j))$\;
        }
        \Return $H_2$\;
    }
    
\end{algorithm} 
The function \texttt{hadLS} remains applicable and unchanged from the previous two-matrix case.

\textit{Initialization.} Let us generalize the SVD-based initialization from Section~\ref{sec:init}. Assume all ranks are equal for simplicity. 
We start with two low-rank matrices: 
\(W_1H_1\) of rank \(r\) that approximates $\sqrt{|X|}$, and \(X'\) of rank \(r' = (p-1)r\) that approximates sign$(X) \circ \sqrt{|X|}$,  using the same idea as in Section~\ref{sec:init}. 
Next, we apply the same initialization method to \(X'\), producing a new rank-\(r\) matrix \(W_2H_2\), and  \(X''\) of rank \(r'' = (p-2)r\), which undergoes further decomposition. This recursive process continues until we obtain \(p\) matrices, each of rank \(r\). At the end of the process, these matrices collectively approximate \(X\).

\section{Numerical experiments} \label{sec:experiments}
In this section, we present a series of numerical experiments to evaluate the performance of the proposed Hadamard decomposition framework whose code is available on GitHub: {\color{blue}\href{https://github.com/WertzSamuel/HadamardDecompositions}{github.com/WertzSamuel/HadamardDecompositions}}. First, we compare our method with the previous algorithm from \cite{b2} designed to solve the same Hadamard decomposition problem, assessing both reconstruction accuracy and computational efficiency. Then, we benchmark our approach against the Singular Value Decomposition (SVD) to highlight its potential for achieving comparable or better approximations while leveraging the structural advantages of the Hadamard model. These comparisons are conducted on both synthetic and real-world datasets to ensure a comprehensive evaluation.

\subsection{Comparison with the previous method from~\cite{b2}} 

The implementation of the previous method from~\cite{b2} was retrieved from the GitHub repository \cite{github}, and is referred to as Alternating Gradient Descent (AGD). It leverages the \texttt{numba} library, which translates Python functions into optimized machine code to achieve performance comparable to that of C. By employing this library, the method efficiently performs a large number of iterations in a relatively short amount of time, with the default number of iterations set to 225,000. To minimize graphs abscises axis length, only one error value out of every 1,000 iterations is stored in the output. The graphs displaying the error evolution with respect to the number of iterations are not entirely comparable, as a single iteration in their implementation corresponds to 1,000 iterations in ours.

\subsubsection{Initialization}

Given that the authors employed a distinct initialization method, we conducted a comparison between both approaches across four different datasets to determine the most effective initialization strategy, which was then applied uniformly to both algorithms. Table \ref{tab:initialization} summarizes the reconstruction realtive errors \(err = \frac{\|X - \tilde{X} \|}{\|X\|}\) obtained with the initial matrices for each method.  
\renewcommand{\arraystretch}{1.2}
\begin{table}[!htbp]
    \begin{center}
    \resizebox{0.45\textwidth}{!}{\begin{tabular}{|c|c|c|c|}
        \hline
        Dataset & Rank & AGD~\cite{b2} & SVD-based  \\
        \hline
        \hline
        \multirow{3}{*}{\textbf{Synthetic data}} & \textbf{r = 10} & 11.083 & 0.429 \\ 
        \cline{2-4} & \textbf{r = 20} & 43.897 & 0.363 \\ 
        \cline{2-4} & \textbf{r = 40} & 175.596 & 0.243 \\ 
        \hline

        \multirow{3}{*}{\textbf{Low-rank synthetic data}} & \textbf{r = 10} & 1.013 & 0.770 \\ 
        \cline{2-4} & \textbf{r = 20} & 1.222 & 0.570 \\ 
        \cline{2-4} & \textbf{r = 40} & 3.014 & 0.317 \\ 
        \hline

        \multirow{3}{*}{\textbf{Cameraman image}} & \textbf{r = 10} & 10.535 & 0.122 \\ 
        \cline{2-4} & \textbf{r = 20} & 42.871 & 0.083 \\ 
        \cline{2-4} & \textbf{r = 40} & 175.030 & 0.049 \\ 
        \hline
        
        \multirow{3}{*}{\textbf{Fotball network}} & \textbf{r = 10} & 21.739 & 0.704 \\ 
        \cline{2-4} & \textbf{r = 20} & 85.049 & 0.529 \\ 
        \cline{2-4} & \textbf{r = 40} & 335.440 & 0.322 \\ 
        \hline
    \end{tabular}}
    \caption{Comparison of relative errors of initializations.}
    \label{tab:initialization}
    \end{center}
\end{table} 
We observe that our SVD-based method provides better initializations than the random approach, while its error decreases as the rank increases (as opposed to random initialization). 

\subsubsection{Synthetic data}

The dataset used for compression is a synthetic low-rank matrix generated by multiplying two matrices $A \in \mathbb{R}^{p \times k}$ and $B \in \mathbb{R}^{k \times q}$ obtained from a stadard normal distribution in order to have $X = AB \in \mathbb{R}^{p \times q}$ of rank $k$. The two datasets used are \(100 \times 100\) with a true rank of 35 and another of dimension \(250 \times 250\) with a true rank of 150. 
Fig.~\ref{fig:comparison_synth_all} presents the final reconstruction errors achieved by the four methods after 300 iterations for our approach and 40,000 iterations for AGD. The number of iterations for the latter was limited due to the rapid increase in computational time as the rank $r$ grew. 

We refer to our proposed methods as follows: 
\(BCD_{\frac{1}{L}}\) when using GD-based updates with the Lipschitz step size, 
\(BCD_{\eta^*}\) when using GD-based updates with the optimal step size, 
and \(BCD^*\) when using the exact resolution method. 

\begin{figure}[!htbp]
    \centering
    \begin{minipage}[t]{0.48\columnwidth}
        \centering
        \begin{tikzpicture}[scale=0.45]
            \begin{axis}[
                xlabel={Factorization rank \textit{r}},
                ylabel = {$\frac{\| X - \tilde{X} \|}{\| X \|}$},
                ymode=log,
                grid=both,
                legend style={at={(0.25, 0.5)}, anchor=north, legend columns=1, font=\LARGE}, 
                tick label style={font=\LARGE},
                xlabel style={font=\huge}, 
                ylabel style={font=\huge}
            ]
                
                \addplot[color=red, mark=o] table [x=x4, y=y4, col sep=space] {Graphes/100_100_synth.txt};
                \addlegendentry{AGD}

                \addplot[color=violet, mark=triangle*] table [x=x1, y=y1, col sep=space] {Graphes/100_100_synth.txt};
                \addlegendentry{\(BCD^*\)}
    
                \addplot[color=blue, style=dashed] table [x=x2, y=y2, col sep=space] {Graphes/100_100_synth.txt};
                \addlegendentry{\(BCD_{\eta^*}\)}

                \addplot[color=green, mark=square*] table [x=x3, y=y3, col sep=space] {Graphes/100_100_synth.txt};
                \addlegendentry{\(BCD_{\frac{1}{L}}\)}

            \end{axis}
        \end{tikzpicture}
        \vspace{0.5em}
        \textbf{(a)} $m$ = $n$ = 100, $r$ = 35.
    \end{minipage}
    \begin{minipage}[t]{0.48\columnwidth}
        \centering
        \begin{tikzpicture}[scale=0.45]
            \begin{axis}[
                xlabel={Factorization rank \textit{r}},
                ylabel = {$\frac{\| X - \tilde{X} \|}{\| X \|}$},
                ymode=log,
                grid=both,
                legend style={at={(0.25, 0.5)}, anchor=north, legend columns=1, font=\LARGE}, 
                tick label style={font=\LARGE},
                xlabel style={font=\huge}, 
                ylabel style={font=\huge}
            ]
                
                \addplot[color=red, mark=o] table [x=x4, y=y4, col sep=space] {Graphes/250_250_synth.txt};
                \addlegendentry{AGD}

                \addplot[color=violet, mark=triangle*] table [x=x1, y=y1, col sep=space] {Graphes/250_250_synth.txt};
                \addlegendentry{\(BCD^*\)}
    
                \addplot[color=blue, style=dashed] table [x=x2, y=y2, col sep=space] {Graphes/250_250_synth.txt};
                \addlegendentry{\(BCD_{\eta^*}\)}

                \addplot[color=green, mark=square*] table [x=x3, y=y3, col sep=space] {Graphes/250_250_synth.txt};
                \addlegendentry{\(BCD_{\frac{1}{L}}\)}

            \end{axis}
        \end{tikzpicture}
        \vspace{0.5em}
        \textbf{(b)} $m$ = $n$ = 250, $r$ = 150.
    \end{minipage}
    \caption{Performance comparison of our methods and AGD~\cite{b2} on synthetic datasets.}
    \label{fig:comparison_synth_all}
\end{figure}
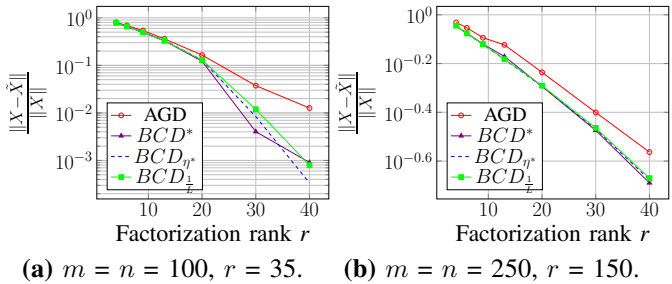

As the results obtained for the three methods are pretty close, we chose to use the exact resolution version for the following experiments because it is a bit more efficient and faster to converge on average particularly for smaller target ranks.  
Next, we allow the methods sufficient time to converge, conducting a second test with the default 225,000 iterations for AGD~\cite{b2} and 225 iterations for \(BCD^*\). We also extended the experiment to a full-rank synthetic dataset generated from standard normal distribution to evaluate the methods under a different scenario. The results are shown in Fig.~\ref{fig:comparison_synth}.

\begin{figure}[!htbp]
    \centering
    \begin{minipage}[t]{0.48\columnwidth}
        \centering
        \begin{tikzpicture}[scale=0.45]
            \begin{axis}[
                xlabel={Iterations},
                ylabel = {$\frac{e(t) - e_{min}}{e(0) - e_{min}}$},
                ymode=log,
                grid=both,
                legend style={at={(0.62, 0.95)}, anchor=north, legend columns=2, font=\LARGE}, 
                tick label style={font=\LARGE},
                xlabel style={font=\huge}, 
                ylabel style={font=\huge}
            ]
                \addplot[color=red] table [x=x1, y=y1, col sep=space] {Graphes/synth_ciaperoni_it.txt};
                \addlegendentry{\(BCD^*\)}
    
                \addplot[color=blue, style=dashed] table [x=x2, y=y2, col sep=space] {Graphes/synth_ciaperoni_it.txt};
                \addlegendentry{AGD}
            \end{axis}
        \end{tikzpicture}
        \vspace{0.5em}
        \textbf{(a)} Full-rank synthetic data
    \end{minipage}
    \begin{minipage}[t]{0.48\columnwidth}
        \centering
        \begin{tikzpicture}[scale=0.45]
            \begin{axis}[
                xlabel={Iterations},
                ylabel = {$\frac{e(t) - e_{min}}{e(0) - e_{min}}$},
                ymode=log,
                grid=both,
                legend style={at={(0.4, 0.2)}, anchor=north, legend columns=2, font=\LARGE}, 
                tick label style={font=\LARGE},
                xlabel style={font=\huge}, 
                ylabel style={font=\huge}
            ]
                \addplot[color=red] table [x=x1, y=y1, col sep=space] {Graphes/low_rank_synth_ciaperoni_it.txt};
                \addlegendentry{\(BCD^*\)}
    
                \addplot[color=blue, style=dashed] table [x=x2, y=y2, col sep=space] {Graphes/low_rank_synth_ciaperoni_it.txt};
                \addlegendentry{AGD}
            \end{axis}
        \end{tikzpicture}
        \vspace{0.5em}
        \textbf{(b)} Low-rank synthetic data
    \end{minipage}
    \caption{Performance comparison of our method and AGD~\cite{b2}  on synthetic datasets. 
    \label{fig:comparison_synth}}
\end{figure}
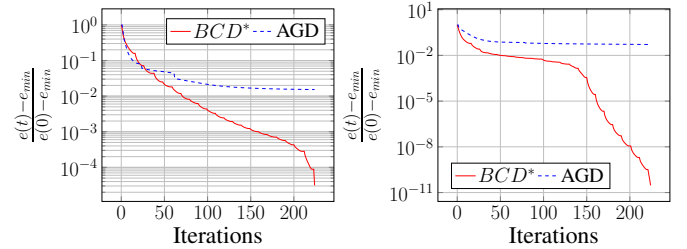

\subsubsection{Real dataset}

For the comparison with AGD~\cite{b2}, we choose two datasets: the cameraman image \(\in \mathbb{R}^{256 \times 256}\) from skimage library and the football netwotk adjacence matrix \(\in \mathbb{R}^{115 \times 115}\) from \cite{b6} in Fig.~\ref{fig:comparison_real1}.
This time, AGD~\cite{b2} is really effective and 225 iterations were not enough to outperform it. So we decided to plot the error in function of the time in order to better show the advantage of our method. 
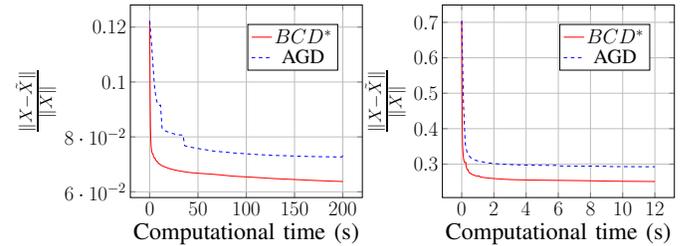
\begin{figure}[!htbp]
    \centering
    \begin{minipage}[t]{0.45\columnwidth}
        \centering
        \begin{tikzpicture}[scale=0.45]
            \begin{axis}[
                xlabel={Computational time (s)},
                ylabel = {$\frac{\| X - \tilde{X} \|}{\| X \|}$},
                grid=both,
                legend style={at={(0.7, 0.9)}, anchor=north, legend columns=1, font=\LARGE}, 
                tick label style={font=\LARGE},
                xlabel style={font=\huge}, 
                ylabel style={font=\huge}
            ]
                
                \addplot[color=red] table [x=x1, y=y1, col sep=space] {Graphes/cam_ciaperoni_time_same.txt};
                \addlegendentry{\(BCD^*\)}
    
                \addplot[color=blue, style=dashed] table [x=x2, y=y2, col sep=space] {Graphes/cam_ciaperoni_time_same.txt};
                \addlegendentry{AGD}
            \end{axis}
        \end{tikzpicture}
        \vspace{0.5em}
        \textbf{(a)} Cameraman image
    \end{minipage}
    \hfill
    \begin{minipage}[t]{0.48\columnwidth}
        \centering
        \begin{tikzpicture}[scale=0.45]
            \begin{axis}[
                xlabel={Computational time (s)},
                ylabel = {$\frac{\| X - \tilde{X} \|}{\| X \|}$},
                grid=both,
                legend style={at={(0.7, 0.9)}, anchor=north, legend columns=1, font=\LARGE}, 
                tick label style={font=\LARGE},
                xlabel style={font=\huge}, 
                ylabel style={font=\huge}
            ]
                
                \addplot[color=red] table [x=x1, y=y1, col sep=space] {Graphes/football_ciaperoni_time_same.txt};
                \addlegendentry{\(BCD^*\)}
    
                \addplot[color=blue, style=dashed] table [x=x2, y=y2, col sep=space] {Graphes/football_ciaperoni_time_same.txt};
                \addlegendentry{AGD}
            \end{axis}
        \end{tikzpicture}
        \vspace{0.5em}
        \textbf{(b)} Football Network
    \end{minipage}
    \caption{Performance comparison of our method and AGD~\cite{b2} on real datasets.}
    \label{fig:comparison_real1}
\end{figure}

The previous graph were obtained with a target rank of 10. In order to verify the performance of our algorithm we tested it for different values of $r$ on the football dataset. The results are presented in Table \ref{tab:football_ciaperoni}. 
\\
The process was interrupted if the relative error did not decrease at least by \(10^{-6}\) for 10 consecutive iterations.

\begin{table}[!htbp]
    \begin{center}
        \resizebox{0.5\textwidth}{!}{
            \begin{tabular}{|c|c|c|c|c|}
                \hline
                \textbf{ }& \multicolumn{2}{|c|}{\textbf{AGD}} & \multicolumn{2}{|c|}{\textbf{BCD}}\\
                \cline{2-5} 
                \textbf{Rank $r$}& \textbf{\textit{Relative error}}& \textbf{\textit{Time (s)}}& \textbf{\textit{Relative error}}& \textbf{\textit{Time (s)}} \\
                \hline
                \hline
                4 & 0.627 & 2.386 & 0.619 & 7.53 \\
                \hline
                6 & 0.508 & 6.47 & 0.495 & 9.59 \\
                \hline
                9  & 0.351 & 8.65 & 0.315 & 12.56 \\ 
                \hline
                13 & 0.155 & 5.30 & 0.066 & 15.83 \\ 
                \hline
                20 & 0.032 & 37.72 & 0.014 & 81.35 \\ 
                \hline
                30 & 4.790\(\times 10^{-3}\) & 37.97 & 2.108\(\times 10^{-3}\) & 125.14 \\
                \hline
                40 & 1.192\(\times 10^{-3}\) & 38.39 & 4.512\(\times 10^{-6}\) & 232.36 \\
                \hline
            \end{tabular}
        }
    \caption{Comparison between AGD and BCD on the Football network dataset.} 
    \label{tab:football_ciaperoni}
    \end{center}
\end{table}

These results show that our algorithm consistently achieved a better minimum across all tests conducted, outperforming AGD~\cite{b2} in each case. This performance validates the efficiency of our approach, allowing us to confidently extend the study to the multi-factor generalization and compare our method to the SVD. 

\subsection{Comparison with the SVD}

We compare our method to the SVD across several datasets. Specifically, we evaluate our method with 2, 3, and 4 low-rank matrices. Additionally, we include two more datasets in our experiments: the Low Resolution Spectrometer dataset \cite{b8}, which belongs to \( \mathbb{R}^{531 \times 101}\) , and the adjacency matrix of the character relationships network from Les Misérables, introduced by D. Knuth in \cite{b7}, which is a matrix in \( \mathbb{R}^{77 \times 77}\). Fig.~\ref{fig:comparison_real2} shows the result. 
The y-axis represents the sum of the ranks across all low-rank matrices. For example, a value of 12 indicates a rank of 12 for SVD, while the Hadamard decomposition used ranks of 6, 4, and 3 for its 2, 3 or 4 low-rank matrices. This approach requires the rank to be divisible by these numbers leading to smaller amount of points on the graphs. The values tested were 12, 24, 36, and 48. 
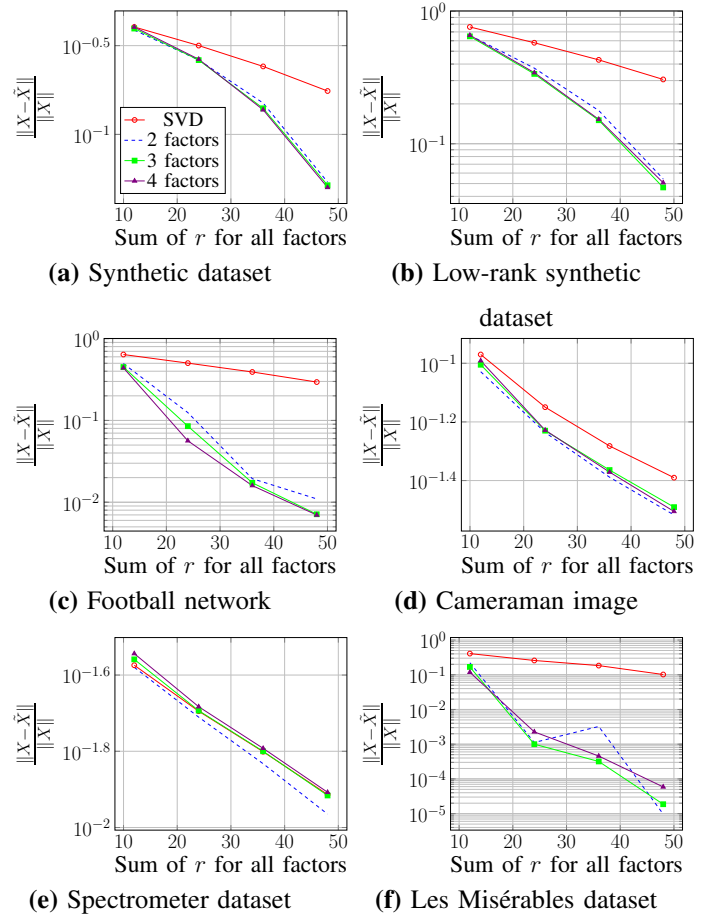
\begin{figure}[!htbp]
    \centering
    \begin{minipage}[t]{0.45\columnwidth}
        \centering
        \begin{tikzpicture}[scale=0.45]
            \begin{axis}[
                xlabel={Sum of \(r\) for all factors},
                ylabel = {$\frac{\| X - \tilde{X} \|}{\| X \|}$},
                ymode=log,
                grid=both,
                legend style={at={(0.25, 0.5)}, anchor=north, legend columns=1, font=\LARGE}, 
                tick label style={font=\LARGE},
                xlabel style={font=\huge}, 
                ylabel style={font=\huge} 
            ]
                
                \addplot[color=red, mark=o] table [x=x4, y=y4, col sep=space] {Graphes/synth_SVD.txt};
                \addlegendentry{SVD}
    
                \addplot[color=blue, style=dashed] table [x=x1, y=y1, col sep=space] {Graphes/synth_SVD.txt};
                \addlegendentry{2 factors}

                \addplot[color=green, mark=square*] table [x=x2, y=y2, col sep=space] {Graphes/synth_SVD.txt};
                \addlegendentry{3 factors}

                \addplot[color=violet, mark=triangle*] table [x=x3, y=y3, col sep=space] {Graphes/synth_SVD.txt};
                \addlegendentry{4 factors}
            \end{axis}
        \end{tikzpicture}
        \vspace{0.5em}
        \textbf{(a)} Synthetic dataset
    \end{minipage}
    \hfill
    \begin{minipage}[t]{0.48\columnwidth}
        \centering
        \begin{tikzpicture}[scale=0.45]
            \begin{axis}[
                xlabel={Sum of \(r\) for all factors},
                ylabel = {$\frac{\| X - \tilde{X} \|}{\| X \|}$},
                ymode=log,
                grid=both,
                tick label style={font=\LARGE},
                xlabel style={font=\huge}, 
                ylabel style={font=\huge}
            ]
    
                \addplot[color=red, mark=o] table [x=x4, y=y4, col sep=space] {Graphes/low_rank_synth_SVD.txt};
    
                \addplot[color=blue, style=dashed] table [x=x1, y=y1, col sep=space] {Graphes/low_rank_synth_SVD.txt};

                \addplot[color=green, mark=square*] table [x=x2, y=y2, col sep=space] {Graphes/low_rank_synth_SVD.txt};

                \addplot[color=violet, mark=triangle*] table [x=x3, y=y3, col sep=space] {Graphes/low_rank_synth_SVD.txt};
            \end{axis}
        \end{tikzpicture}
        \vspace{0.5em}
        \textbf{(b)} Low-rank synthetic dataset
    \end{minipage} 
    \begin{minipage}[t]{0.45\columnwidth}
        \centering
        \begin{tikzpicture}[scale=0.45]
            \begin{axis}[
                xlabel={Sum of \(r\) for all factors},
                ylabel = {$\frac{\| X - \tilde{X} \|}{\| X \|}$},
                ymode=log,
                grid=both,
                tick label style={font=\LARGE},
                xlabel style={font=\huge}, 
                ylabel style={font=\huge}
            ]
                
                \addplot[color=red, mark=o] table [x=x4, y=y4, col sep=space] {Graphes/foot_SVD.txt};
    
                \addplot[color=blue, style=dashed] table [x=x1, y=y1, col sep=space] {Graphes/foot_SVD.txt};

                \addplot[color=green, mark=square*] table [x=x2, y=y2, col sep=space] {Graphes/foot_SVD.txt};

                \addplot[color=violet, mark=triangle*] table [x=x3, y=y3, col sep=space] {Graphes/foot_SVD.txt};
            \end{axis}
        \end{tikzpicture}
        \vspace{0.5em}
        \textbf{(c)} Football network
    \end{minipage}
    \hfill
    \begin{minipage}[t]{0.48\columnwidth}
        \centering
        \begin{tikzpicture}[scale=0.45]
            \begin{axis}[
                xlabel={Sum of \(r\) for all factors},
                ylabel = {$\frac{\| X - \tilde{X} \|}{\| X \|}$},
                ymode=log,
                grid=both,
                tick label style={font=\LARGE},
                xlabel style={font=\huge}, 
                ylabel style={font=\huge}
            ]
                
                \addplot[color=red, mark=o] table [x=x4, y=y4, col sep=space] {Graphes/cam_SVD.txt};
    
                \addplot[color=blue, style=dashed] table [x=x1, y=y1, col sep=space] {Graphes/cam_SVD.txt};

                \addplot[color=green, mark=square*] table [x=x2, y=y2, col sep=space] {Graphes/cam_SVD.txt};

                \addplot[color=violet, mark=triangle*] table [x=x3, y=y3, col sep=space] {Graphes/cam_SVD.txt};
            \end{axis}
        \end{tikzpicture}
        \vspace{0.5em}
        \textbf{(d)} Cameraman image
    \end{minipage}

    \begin{minipage}[t]{0.45\columnwidth}
        \centering
        \begin{tikzpicture}[scale=0.45]
            \begin{axis}[
                xlabel={Sum of \(r\) for all factors},
                ylabel = {$\frac{\| X - \tilde{X} \|}{\| X \|}$},
                ymode=log,
                grid=both,
                tick label style={font=\LARGE},
                xlabel style={font=\huge}, 
                ylabel style={font=\huge}
            ]
                
                \addplot[color=red, mark=o] table [x=x4, y=y4, col sep=space] {Graphes/spectro_SVD.txt};
    
                \addplot[color=blue, style=dashed] table [x=x1, y=y1, col sep=space] {Graphes/spectro_SVD.txt};

                \addplot[color=green, mark=square*] table [x=x2, y=y2, col sep=space] {Graphes/spectro_SVD.txt};

                \addplot[color=violet, mark=triangle*] table [x=x3, y=y3, col sep=space] {Graphes/spectro_SVD.txt};
            \end{axis}
        \end{tikzpicture}
        \vspace{0.5em}
        \textbf{(e)} Spectrometer dataset
    \end{minipage}
    \hfill
    \begin{minipage}[t]{0.48\columnwidth}
        \centering
        \begin{tikzpicture}[scale=0.45]
            \begin{axis}[
                xlabel={Sum of \(r\) for all factors},
                ylabel = {$\frac{\| X - \tilde{X} \|}{\| X \|}$},
                ymode=log,
                grid=both,
                tick label style={font=\LARGE},
                xlabel style={font=\huge}, 
                ylabel style={font=\huge}
            ]
                
                \addplot[color=red, mark=o] table [x=x4, y=y4, col sep=space] {Graphes/miserables_SVD.txt};
    
                \addplot[color=blue, style=dashed] table [x=x1, y=y1, col sep=space] {Graphes/miserables_SVD.txt};

                \addplot[color=green, mark=square*] table [x=x2, y=y2, col sep=space] {Graphes/miserables_SVD.txt};

                \addplot[color=violet, mark=triangle*] table [x=x3, y=y3, col sep=space] {Graphes/miserables_SVD.txt};
            \end{axis}
        \end{tikzpicture}
        \vspace{0.5em}
        \textbf{(f)} Les Misérables dataset
    \end{minipage}
    \caption{Performance comparison of our method for different numbers of factors and SVD.}
    \label{fig:comparison_real2}
\end{figure} 
The Hadamard decomposition method outperforms the SVD for the same number of parameters, particularly on sparse matrices such as the football network and the Les Misérables adjacency matrices, which have densities of only 0.093\% and 0.086\%, respectively.

\section{Conclusion} \label{sec:concl}

In this study, we developed efficient algorithms using BCD for the Hadamard decomposition. Leveraging an alternating optimization strategy, our algorithm effectively decomposes the global non-convex problem into tractable, unconstrained least squares, sub-problems. We further improved performance using a clever SVD-based initialization and extended the framework to handle more than two low-rank matrices, thereby enabling approximations with higher effective ranks (as illustrated by Theorem~\ref{th:identity}) while maintaining computational efficiency. 
Our experiments showed that our proposed algorithms outperformed a gradient-based method from~\cite{b2}, while allowing to obtain significantly lower reconstruction errors than the SVD, particularly for sparse datasets, which illustrates that the Hadamard decomposition is more expressive. 


\appendix

In this appendix, we first provide the maximum possible achievable rank by a Hadamard decomposition with budget $R$. Then we prove Theorem~\ref{th:identity} which attains the maximum possible rank for the identity matrix. Finally, we report some numerical experiments of our algorithm applied on the identity matrix to see whether it is able to recover this solution. 

\subsection{Maximum rank for a given budget}
\label{sec:appendixA}

Given a Hadamard decomposition with $p$ low-rank matrices, $X = (W_1 H_1) \circ \dots (W_p H_p)$, we address the following question: what is the maximum possible achievable rank of $X$ for a given budget $R = \sum_{i=1}^p r_i$ where $r_i$ is the inner dimension of $W_i H_i$? 

Mathematically, given a natural $R>1$, among all the combinations of positive numbers that adds up to $R$, we are looking for the one that maximizes the product of these numbers. More formally, we want to identify the length $p$ and the entries of a vector $r$ of positive natural numbers solving:
\begin{equation} \label{eq:maxrank}
    \max_{r_i \in \mathbb N, p} \; \prod_{i=1}^p r_i \quad \text{such that }\quad \sum_{i=1}^pr_i=R.
\end{equation}
Note that $p$ is part of the optimization: we can choose the length of the vector $r$ to maximize the objective. 

\begin{lemma} \label{lem:lemma1} For $R > 1$, An optimal solution of~\eqref{eq:maxrank} is given by:
    \begin{itemize}
        \item 
        $r=\{\underbrace{3,\hdots,3}_{k \text{ times}}\}$ with $p=k$ when $R=3k$ with $k\in\mathbb{N}$,
        \item 
        $r=\{2,2,\underbrace{3,\hdots,3}_{k-1 \text{ times}}\}$
        with $p=k+1$ when  $R=3k+1$ with $k\in\mathbb{N}$,
        \item 
        $r=\{2,\underbrace{3,\hdots,3}_{k \text{ times}}\}$
        with $p=k+1$ when  $R=3k+2$ with $k\in\mathbb{N}$.
    \end{itemize}
\end{lemma}

\begin{proof} 

    Let us show the following facts which will imply the result.  
    
    \begin{enumerate}
        \item $r_i\neq 1$ for all $i$: 
        Since $R>1$, if the $j$th entry of $r$ is equal to one, that is, $r_j=1$, it means there is at least another entry $r_k \geq 1$, with $k\neq j$.
        By removing the $j$th entry of $r$ (reducing $p$ by one) and increasing the $k$th entry of $r$ by one, we obtain a better solution, since $1 * r_k < r_k+1$.

    \item  $r_i\leq 4$ for all $i$:         
        Suppose the $j$th entry of $r$ is such that $r_j=4+K$ with $K\geq 1$. 
        Then it is possible to create another admissible solution by setting $r_j=1+K$ and by adding an entry $r_{p+1}=3$.
        This solution is better since  
        $(4+K) < 3(K+1)$        for any $K\geq1$. 
        
        \item There exists an optimal solution with $r_i\in\{2,3\}$ for all~$i$: 
        By 1) and 2) above, and the facts that $2+2=4$ and $2^2=4$, all the entries equal to $4$ can be replaced by two entries equal to $2$. 
        
        \item In an optimal solution with $r_i\in\{2,3\}$ for all $i$, there are at most two entries equal to 2:   
        Since $2+2+2=3+3$ and $2^3<3^2$, it is not optimal to have more than two $2$'s. 
    
    \end{enumerate}
\end{proof}

\subsection{Construction of an Hadamard decomposition of the identity} 
\label{sec:appendixB}

Before showing properly how to construct the factorization of the identity, let us describe the core argument with the following lemma.  

\begin{lemma}\label{lem:lemma2} Let the $m$-by-$n$ matrix $A$ be 
 the Hadamard product of $p$ matrices such that $A=X_1\circ \dots \circ X_p$ where $X_i=W_iH_i$ for all $i$. For any natural $r$, it is possible to write the  $mr$-by-$nr$ matrix $I_r\kronecker A$, where $\kronecker$ is the Kronecker product, as the Hadamard product of $p+1$ matrices $X'_i$, that is, 
 \begin{equation}   \label{eq:krondec}
I_r\kronecker A = 
X'_1\circ \dots \circ X'_p \circ X'_{p+1}, 
 \end{equation}
 where $\rank(X'_i)=\rank(X_i)$ for all $i=1,\dots,p$ and $\rank(X'_{p+1})=r$. 
\end{lemma}
 \begin{proof} For all $i$, let $X'_i = 1_{r\times r}\kronecker X_i = W_i'H_i'$ with 
 $1_{r\times r}$ the all-one $r$-by-$r$ matrix, 
 $W_i'=1_{r \times 1} \kronecker W_i$ and $H_i'= 1_{1 \times r} \kronecker H_i$ such that $\text{rank}(X'_i)=\text{rank}(X_i)$.
By multiplying elementwise $X'_i$ with the matrix $X'_{p+1}=I_r\kronecker 1_{m\times n}$ of rank $r$, we obtain a block diagonal matrix of size $mr$-by-$nr$ where the block $X_i$ is repeated $r$ times on the diagonal.
Hence \eqref{eq:krondec} holds. 
\end{proof}

By combining Lemma~\ref{lem:lemma1} and Lemma~\ref{lem:lemma2}, we can prove Theorem~\ref{th:identity}.  
\begin{proof}[Proof of Theorem~\ref{th:identity}] 
Since $\text{rank}(A\circ B) \leq \text{rank}(A)\text{rank}(B)$, the rank of $X_1\circ X_2 \circ \hdots \circ X_p$ can be as much as $\prod_{i=1}^p\text{rank}(X_i)$, that is $n\leq \prod_{i=1}^p\text{rank}(X_i)$.
Denoting $r_i=\text{rank}(X_i)$ and considering that $\sum_{i=1}^pr_i=R$, the maximum value of $\prod_{i=1}^p r_i$ is then given by the optimal solution of the problem described in Lemma~\ref{lem:lemma1}.
Using $X_i=I_2$ or $X_i=I_3$, and by using recursively the construction described in Lemma~\ref{lem:lemma2}, it is possible to reach this maximum value, and construct the identity matrix. 
\end{proof}

\subsection{Numerical results on the  decomposition of the identity}


In this section, we test if our algorithm is able to retrieve the exact decomposition for different sizes of identity matrices. We run the algorithm hundred times with random initialization matrices obtained by sampling the standard normal distribution and computed the number of times the relative error was lower than \(10^{-5}\). 
Table \ref{tab:identity} presents the results obtained. 





\begin{table}[!htbp]
    \begin{center}
    \begin{tabular}{|c|c|c|>{\centering\arraybackslash}p{1.5cm}|c|}
        \hline
        $R$ & $n$ & $p$ & \textbf{Percentage of success} & \textbf{Average relative error} \\
        \hline
        \hline
        6 & 9 & 2 & 96\% & $1.770 \times 10^{-5} \pm 1.650 \times 10^{-4}$\\
        \hline
        7 & 12 & 3 & 95\% & $0.005776 \pm 0.04041$\\
        \hline
        8 & 18 & 3 & 70\% & $0.02122 \pm 0.06745$\\
        \hline
        9 & 27 & 3 & 56\% & $0.0625 \pm 0.0912$\\
        \hline
        10 & 36 & 4 & 8\% & $0.1755 \pm 0.08521$ \\
        \hline
        11 & 54 & 4 &  0\% & $0.2351 \pm 0.05371$\\
        \hline
        12 & 55 & 4 &  6\% & $0.08679 \pm 0.09002$ \\
        \hline
        12 & 81 & 4 &  0\% & $0.309 \pm 0.042$\\
        \hline
        16 & 81 & 6 &  13\% & $0.1550 \pm 0.0747$ \\
        \hline
    \end{tabular}
    \caption{Percentage of success on decomposing the $n$-by-$n$ identity matrix for a budget $R$ with $p$ low-rank matrices.}
    \label{tab:identity}
    \end{center}
\end{table}

Our algorithm achieves good reconstruction performance on small identity matrices, up to $n = 27$, with more than 50\% solution with error below \(10^{-5}\). As the dimension increases, the algorithm has a harder time to find exact decompositions, e.g., it never finds a global optimum our of the 100 initializations for $n=54, 81$. This is due to the non-convexity of the Hadamard decomposition problem.  However, if we increase the budget (which overparametrizes the set of solutions), it can find exact decompositions.


\end{document}